\documentclass{article} 
\usepackage{nips12submit_e,times}

\usepackage{amsmath}
\usepackage{amsthm}
\usepackage{amsfonts}

\usepackage{ifthen}
\newboolean{showcomments}
\setboolean{showcomments}{false} 
\ifthenelse{\boolean{showcomments}}
{\newcommand{\nb}[2]{\fbox{\bfseries\small\sffamily#1}{\sf\small$\triangleright$\textbf{#2}$\triangleleft$}}}
{\newcommand{\nb}[2]{}}

\usepackage{pgf}
\usepackage{tikz}
\usetikzlibrary{arrows,automata,positioning}

\title{On the Use of Non-Stationary Policies for Stationary Infinite-Horizon Markov Decision Processes}

\author{
Bruno Scherrer \\
Inria, Villers-l\`es-Nancy, F-54600, France \\
\texttt{bruno.scherrer@inria.fr} \\
\And
Boris Lesner \\
Inria, Villers-l\`es-Nancy, F-54600, France \\
\texttt{boris.lesner@inria.fr} 
}

\newtheorem{thm}{Theorem}
\newtheorem{lemma}{Lemma}

\newtheorem{exmp}{Example}

\newcommand{\greedy}[1]{\ensuremath{\mathcal G(#1)}}

\newcommand{\norm}[1]{\ensuremath{{\lVert #1 \rVert}}}
\newcommand{\len}[1]{\ensuremath{m}}

\nipsfinalcopy 

\begin{document}

\maketitle

\begin{abstract}
  We consider infinite-horizon stationary $\gamma$-discounted Markov
  Decision Processes, for which it is known that there exists a
  stationary optimal policy. Using Value and Policy Iteration with
  some error $\epsilon$ at each iteration, it is well-known that one
  can compute stationary policies that are
  $\frac{2\gamma}{(1-\gamma)^2}\epsilon$-optimal. After arguing that
  this guarantee is tight, we develop variations of Value and Policy
  Iteration for computing non-stationary policies that can be up to
  $\frac{2\gamma}{1-\gamma}\epsilon$-optimal, which constitutes a
  significant improvement in the usual situation when $\gamma$ is
  close to $1$. Surprisingly, this shows that the problem of
  ``computing near-optimal non-stationary policies'' is much simpler
  than that of ``computing near-optimal stationary policies''.
\end{abstract}

\section{Introduction}

Given an infinite-horizon stationary $\gamma$-discounted Markov Decision Process~\cite{puterman,ndp}, we consider
approximate versions of the standard Dynamic Programming algorithms, Policy and Value Iteration, that build sequences of value functions $v_k$ and
policies $\pi_k$ as follows
\begin{align}
\mbox{Approximate Value Iteration (AVI):~~~~~}& ~~~~~~v_{k+1} ~ \gets~ T v_k + \epsilon_{k+1}  \label{avi} \\
 \mbox{Approximate Policy Iteration (API):~~~~~} & 
\left\{\begin{array}{rcl}
v_{k} &\gets &v_{\pi_{k}} + \epsilon_{k} \\
\pi_{k+1} &\gets &\mbox{any element of } \greedy{v_k}
\end{array}\right.
\label{api}
\end{align}
where $v_0$ and $\pi_0$ are arbitrary,  $T$ is the Bellman optimality operator, $v_{\pi_{k}}$ is the value of policy
$\pi_k$ and $\greedy{v_k}$ is the set of policies that are
greedy with respect to $v_k$. At each iteration $k$, the term $\epsilon_k$ accounts for a possible approximation of the Bellman operator (for AVI) or the
evaluation of $v_{\pi_k}$ (for API). Throughout the paper, we will assume that error terms $\epsilon_k$ satisfy for all $k$,
$\norm{\epsilon_k}_\infty\leq \epsilon$ for some $\epsilon \ge 0$. 
Under this assumption, it is well-known that both algorithms share the
following performance bound (see \cite{singh94,Gordon95,ndp} for AVI and \cite{ndp} for API):
\begin{thm}
\label{thm:classic-bound}
For API (resp. AVI), the \emph{loss} due to running policy $\pi_k$  (resp. any policy $\pi_k$ in $\greedy{v_{k-1}}$) instead of the optimal policy $\pi_*$ satisfies
\[
\limsup_{k\to \infty}\norm{v_* - v_{\pi_k}}_\infty \leq \frac {2\gamma}{(1-\gamma)^2}\epsilon.
\]
\end{thm}

The constant $\frac{2\gamma}{(1-\gamma)^2}$ can be very big, in
particular when $\gamma$ is close to $1$, and consequently the above
bound is commonly believed to be conservative for practical applications.
Interestingly, this very constant $\frac{2\gamma}{(1-\gamma)^2}$ appears in many works analyzing AVI
algorithms~\cite{singh94,Gordon95,TsitsiklisR96,guestrin2001,Guestrinjair2003,pineau2003point,Even-dar05y.:planning,ernst2005tree,Munos_SIAM07,Munos_JMLR08,petrik:2008,FaMuSz10},
API
algorithms~\cite{KakadeL02,munos2003,lagoudakis2003least,antos2008learning,farahmand2009regularized,acml2010,busoniu2011least-squares,Lazaric_JMLR2011_a,gabillon:hal-00644935,bertsekas2011,FaMuSz10,gheshlaghidpp}
and in one of their generalization~\cite{thierylpi},
suggesting that it cannot be improved.  Indeed, the bound (and the
$\frac{2\gamma}{(1-\gamma)^2}$ constant) are tight for
API~\cite[Example 6.4]{ndp}, and we will show in
Section~\ref{sec:tightness} -- to our knowledge, this has never been
argued in the literature -- that it is also tight for AVI. 

Even though the theory of optimal control states that there exists a
stationary policy that is optimal, the main contribution of our paper
is to show that looking for a \emph{non-stationary} policy (instead of
a stationary one) may lead to a much better performance bound. In
Section~\ref{sec:avi}, we will show how to deduce such a non-stationary
policy from a run of AVI. In Section~\ref{sec:api}, we
will describe two original policy iteration variations that compute
non-stationary policies.  For all these algorithms, we will prove that
we have a performance bound that can be reduced down to $\frac{2\gamma}{1-\gamma}\epsilon$.
This is a factor $\frac{1}{1-\gamma}$ better than the standard bound of Theorem~\ref{thm:classic-bound},
which is significant when $\gamma$ is close to $1$. 
Surprisingly, this will show that the problem of ``computing near-optimal non-stationary policies'' is much simpler than that
  of ``computing near-optimal stationary policies''.
Before we present these contributions, the next section
begins by precisely describing our setting.

\section{Background}
\label{sec:background}

We consider an infinite-horizon discounted Markov Decision Process~\cite{puterman,ndp} $(\mathcal S, \mathcal
A, P, r,\gamma)$, where $\mathcal S$ is a possibly infinite state
space, $\mathcal A$ is a finite action space, $P(ds'|s,a)$, for all
$(s,a)$, is a probability kernel on $\mathcal S$, $r : \mathcal
S\times\mathcal A \to \mathbb{R}$ is a reward function bounded in max-norm by $R_{\mathrm{max}}$, and $\gamma \in (0,1)$ is a discount factor. A
stationary deterministic policy $\pi:\mathcal S\to\mathcal A$ maps
states to actions. We write $r_\pi(s) = r(s,\pi(s))$ and $P_\pi(ds'|s)=P(ds'|s,\pi(s))$ for the immediate reward and the
stochastic kernel associated to policy $\pi$. The value $v_\pi$ of a
policy $\pi$ is a function mapping states to the expected discounted sum
of rewards received when following $\pi$ from any state: for all
$s\in\mathcal S$, 
$$
v_\pi(s) = \mathbb{E}\left[\sum_{t=0}^\infty \gamma^t
  r_\pi(s_t)\middle|s_0=s, s_{t+1}\sim
  P_\pi(\cdot|s_t)\right].
$$
The value $v_\pi$ is clearly
bounded by $V_{\mathrm{max}} = R_{\mathrm{max}}/(1-\gamma)$.
It is well-known that $v_\pi$ can be characterized as 
the unique fixed point of the linear Bellman operator
associated to a policy $\pi$: $T_\pi:v \mapsto r_\pi + \gamma P_\pi v$.
Similarly, the Bellman optimality operator $T:v \mapsto \max_\pi T_\pi v$ has as unique fixed point the optimal value $v_*=\max_\pi v_\pi$.  A policy $\pi$ is greedy w.r.t. a value
function $v$ if $T_\pi v = T v$, the set of such greedy policies is
written $\greedy{v}$. Finally, a policy $\pi_*$ is optimal, with value
$v_{\pi_*}=v_*$, iff $\pi_*\in\greedy{v_*}$, or equivalently
$T_{\pi_*}v_* = v_*$. 

Though it is known~\cite{puterman,ndp} that there
always exists a deterministic stationary policy that is optimal, we will, in this article, consider non-stationary policies and  now introduce related notations.
Given a sequence $\pi_1, \pi_2,\dots,\pi_k$ of $k$
  stationary policies (this sequence will be clear in the context we describe later), and for any $1 \le m \le k$, we will denote $\pi_{k,m}$ the \emph{periodic non-stationary policy} that takes the first action according to $\pi_k$, the second according to $\pi_{k-1}$, \dots, the $m^{th}$ according to $\pi_{k-m+1}$ and then starts again. Formally, this can be written as
  \[
  \pi_{k,m} = \pi_{k}\ \pi_{k-1}\ \cdots\ \pi_{k-m+1}\ \pi_k\ \pi_{k-1}\ \cdots \pi_{k-m+1} \cdots
  \]
It is straightforward to show that the value $v_{\pi_{k,m}}$ of this periodic non-stationary policy $\pi_{k,m}$ is the unique fixed point of the following operator:
\begin{align*} 
T_{k,m}=T_{\pi_k} ~ T_{\pi_{k-1}}  ~ \cdots ~ T_{\pi_{k-m+1}}.
\end{align*}
Finally, it will be convenient to introduce the following discounted kernel:
\begin{align*}
\Gamma_{k,m}=(\gamma P_{\pi_k}) (\gamma P_{\pi_{k-1}}) \cdots (\gamma P_{\pi_{k-m+1}}).
\end{align*}
In particular, for any pair of values $v$ and $v'$, it can easily be seen that $T_{k,m}v-T_{k,m}v'=\Gamma_{k,m}(v-v')$.

\section{Tightness of the performance bound of Theorem~\ref{thm:classic-bound}}
\label{sec:tightness}

The bound of Theorem~\ref{thm:classic-bound} is tight for API in the sense that there exists an MDP~\cite[Example 6.4]{ndp} for which the bound is reached. To the best of our knowledge, 
a similar argument has never been provided for AVI in the literature. It turns out that the MDP that is used for showing the tightness for API also applies to AVI. This is what we show in this section.
\begin{figure}
\begin{center}
\begin{tikzpicture}[->,>=stealth',shorten >=1pt,auto,node distance=2.8cm,
                    semithick]
\tikzstyle{every state}=[thick]

\node[state] (S1) {$1$};
\node[state] (S2) [right of=S1] {$2$};
\node[state] (S3) [right of=S2] {$3$};
\node        (S_) [right=1cm of S3] {$\quad\dots\quad$};
\node[state] (Sk) [right=1cm of S_] {$k$};
\node        (Sn) [right=1cm of Sk] {$\quad\dots\quad$};

\path (S1) edge [loop above] node {$0$} (S1)
      (S2) edge              node {$0$} (S1)
           edge [loop above] node {$-2\gamma\epsilon$} (S2)
      (S3) edge              node {$0$} (S2)
           edge [loop above] node {$-2(\gamma+\gamma^2)\epsilon$} (S3)
      (S_) edge              node {$0$} (S3)
      (Sk) edge              node {$0$} (S_)
           edge [loop above] node {$-2\frac{\gamma-\gamma^k}{1-\gamma}\epsilon$} (Sk)
      (Sn) edge              node {$0$} (Sk);
\end{tikzpicture}
\end{center}
\caption{The determinisitic MDP for which the bound of Theorem~\ref{thm:classic-bound} is tight for Value and Policy Iteration.}
\label{fig:tight}
\end{figure}
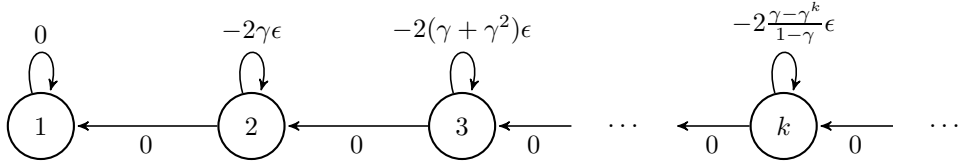
\begin{exmp} 
  \label{exmp:tight}
  Consider the $\gamma$-discounted deterministic MDP from~\cite[Example 6.4]{ndp} depicted on
  Figure~\ref{fig:tight}. It involves states $1,2,\dots$. In state $1$
  there is only one self-loop action with zero reward, for each state $i
  > 1$ there are two possible choices: either \emph{move} to state $i-1$ with
  zero reward or \emph{stay} with reward $r_i =
  -2\frac{\gamma-\gamma^{i}}{1-\gamma}\epsilon$ with $\epsilon \geq
  0$. Clearly the optimal policy in all states $i>1$ is to move to
  $i-1$ and the optimal value function $v_*$ is $0$ in all states.
  \par
  Starting with $v_0 = v_*$, we are going to show that for all iterations
  $k \ge 1$ it is possible to have a policy $\pi_{k+1} \in \greedy{v_k}$ which moves in
  every state but $k+1$ and thus is such that  $v_{\pi_{k+1}}(k+1)=\frac{r_{k+1}}{1-\gamma} =
  -2\frac{\gamma-\gamma^{k+1}}{(1-\gamma)^2}\epsilon$, which meets the bound
  of Theorem~\ref{thm:classic-bound} when $k$ tends to infinity.
  \par
  To do so, we assume that the following approximation errors are made at each
  iteration $k>0$:
  \[
  \epsilon_k(i) = \left\{\begin{array}{ll}
      -\epsilon &\text{\rm if } i=k\\
      \epsilon & \text{\rm if } i=k+1\\
      0&\text{\rm otherwise}
    \end{array}\right..
  \]
  With this error, we are now going to prove by induction on $k$ that for all $k \ge 1$, 
  \[
  v_k(i) = \left\{\begin{array}{ll}
    -\gamma^{k-1}\epsilon &\text{\rm if } i<k\\
    r_k/2-\epsilon & \text{\rm if } i=k\\
    -(r_k/2-\epsilon) & \text{\rm if } i=k+1\\
    0&\text{\rm otherwise}
  \end{array}\right..
  \]
  
Since $v_0 = 0$ the best action
  is clearly to move in every state $i \ge 2$ which gives $v_1 = v_0 +
  \epsilon_1 = \epsilon_1$ which establishes the claim for $k=1$.
  \par
  Assuming that our induction claim holds for $k$, we now show that it also holds for $k+1$.  

\par
For the \emph{move} action, write $q^{\mathrm{m}}_k$ its action-value function.
For all $i>1$ we have $q^{\mathrm{m}}_k(i) = 0+\gamma v_k(i-1)$, hence 
\[
q^{\mathrm{m}}_k(i) = \left\{\begin{array}{lll}
    \gamma(-\gamma^{k-1}\epsilon)&=-\gamma^k\epsilon & \text{if } i=2,\dots,k\\
     \gamma(r_k/2-\epsilon)&=r_{k+1}/2 &\text{if } i=k+1\\
    -\gamma(r_k/2-\epsilon)&=-r_{k+1}/2 & \text{if } i = k+2\\
    0 & &\text{otherwise}
\end{array}\right..
\]
\par
For the \emph{stay} action, write $q^{\mathrm{s}}_k$ its action-value
function. 
For all $i>0$ we have $q^{\mathrm{s}}_k(i) = r_i+\gamma v_k(i)$, hence 
\[
q^{\mathrm{s}}_k(i) = \left\{\begin{array}{lll}
    r_i+\gamma(-\gamma^{k-1}\epsilon) &= r_i-\gamma^k\epsilon& \text{if } i=1,\dots,k-1\\
    r_k+\gamma(r_k/2-\epsilon)&= r_k+r_{k+1}/2&\text{if } i=k\\
    r_{k+1}-r_{k+1}/2&=r_{k+1}/2 & \text{if } i = k+1\\
    r_{k+2}+\gamma 0 &=r_{k+2} & \text{if } i = k+2\\
    0 & &\text{otherwise}
\end{array}\right..
\]
\par
First, only the \emph{stay} action is available in state $1$, hence, since
$r_0=0$ and $\epsilon_{k+1}(1)=0$, we have $v_{k+1}(1) = q^{\mathrm{s}}_k(1) +\epsilon_{k+1}(1)= -\gamma^k\epsilon$,
as desired.  Second, since $r_i < 0$ for all $i>1$ we have
$q^{\mathrm{m}}_k(i) > q^{\mathrm{s}}_k(i)$ for all these states but
$k+1$ where $q^{\mathrm{m}}_k(k+1) = q^{\mathrm{s}}_k(k+1) =
r_{k+1}/2$. Using the fact that $v_{k+1} = \max (q^{\mathrm{m}}_k,
q^{\mathrm{s}}_k) + \epsilon_{k+1}$ gives the result for $v_{k+1}$.
\par
The fact that for $i>1$ we have $q^{\mathrm{m}}_k(i) \geq q^{\mathrm{s}}_k(i)$
with equality only at $i=k+1$ implies that there exists a policy
$\pi_{k+1}$ greedy for $v_k$ which takes the optimal \emph{move} action
in all states but $k+1$ where the \emph{stay} action has the same value,
leaving the algorithm the possibility of choosing the suboptimal
\emph{stay} action in this state, yielding a value $v_{\pi_{k+1}}(k+1)$,
matching the upper bound as $k$ goes to infinity.
\end{exmp}
\par
Since Example~\ref{exmp:tight} shows that the bound of
Theorem~\ref{thm:classic-bound} is tight, improving performance bounds imply to modify the algorithms. The following sections of the paper shows that considering non-stationary policies instead of stationary policies is an interesting path to follow.

\section{Deducing a non-stationary policy from AVI}
\label{sec:avi}

While AVI (Equation~\eqref{avi}) is usually considered as generating
a sequence of values $v_0,v_1,\dots,v_{k-1}$, it also implicitely produces
a sequence\footnote{A given sequence of value functions may induce many
  sequences of policies since more than one greedy policy may exist for
  one particular value function. Our results holds for all such possible
  choices of greedy policies.} of policies
$\pi_1,\pi_2,\dots,\pi_{k}$, where for $i=0,\dots,k-1$,
$\pi_{i+1}\in\greedy{v_i}$. Instead of outputing only the last
policy $\pi_{k}$, we here simply propose to output the periodic non-stationary
policy $\pi_{k,m}$ that loops over the last $m$ generated policies. The following theorem shows that it is indeed a good idea.
\begin{thm} 
\label{thm:avi}
For all iteration $k$ and $m$ such that $1 \le m \le k$, the loss of running the non-stationary policy
$\pi_{k,m}$ instead of the optimal policy $\pi_*$ satisfies:
\[
\norm{v_*-v_{\pi_{k,m}}}_\infty \le \frac{2}{1-\gamma^m} \left(
    \frac{\gamma-\gamma^k}{1-\gamma} \epsilon + \gamma^k \norm{v_*-v_0}_{\infty}
  \right).
  \]
\end{thm}
When $m=1$ and $k$ tends to infinity, one exactly recovers the result of
Theorem~\ref{thm:classic-bound}. For general $m$, this new bound is a
factor $\frac{1-\gamma^m}{1-\gamma}$ better than the standard  bound of Theorem~\ref{thm:classic-bound}. The choice
that optimizes the bound, $m=k$, and  which consists in looping over all the policies generated \emph{from the very start}, leads to the following bound:
\[
  \norm{v_*-v_{\pi_{k,k}}}_\infty \le 2\left(
    \frac{\gamma}{1-\gamma}-\frac{\gamma^k}{1-\gamma^k} \right) \epsilon
  + \frac{2\gamma^k}{1-\gamma^k}\norm{v_*-v_0}_\infty,
\]
that tends to $\frac{2\gamma}{1-\gamma}\epsilon$ when $k$ tends to
$\infty$. 

The rest of the section is devoted to the proof of Theorem~\ref{thm:avi}.
An important step of our proof lies in the following lemma, that implies that for sufficiently big $m$, $v_k=T v_{k-1}+\epsilon_k$ is a rather good approximation (of the order $\frac{\epsilon}{1-\gamma}$) of the value $v_{\pi_{k,m}}$ of the non-stationary policy $\pi_{k,m}$ (whereas in general, it is a much poorer approximation of the value $v_{\pi_k}$ of the last stationary policy $\pi_k$).
\begin{lemma}
\label{lem:vi-norm-rec}
For all $m$ and $k$ such that $1 \le m \le k$, 
\begin{align}
  \|T v_{k-1}-v_{\pi_{k,m}}\|_\infty &\leq \gamma^m \|v_{k-m}-
  v_{\pi_{k,m}} \|_\infty + \frac{\gamma-\gamma^m}{1-\gamma}
  \epsilon. \nonumber
\end{align}
\end{lemma}
\begin{proof}[Proof of Lemma~\ref{lem:vi-norm-rec}]
The value of $\pi_{k,m}$ satisfies:
\begin{align}
\label{eq:vi-vpikm}
v_{\pi_{k,m}} = T_{\pi_k} T_{\pi_{k-1}} \cdots T_{\pi_{k-m+1}} v_{\pi_{k,m}}.
\end{align}
By induction, it can be shown that the sequence of values generated by
AVI satisfies:
\begin{equation}
\label{eq:vi-tpikvk1}
T_{\pi_k} v_{k-1} = T_{\pi_{k}} T_{\pi_{k-1}} \cdots T_{\pi_{k-m+1}}
v_{k-m} + \sum_{i=1}^{m-1} \Gamma_{k,i} \epsilon_{k-i}.
\end{equation}
By substracting Equations~\eqref{eq:vi-tpikvk1} and \eqref{eq:vi-vpikm}, one obtains:
\begin{align*}
  T v_{k-1} - v_{\pi_{k,m}} =  T_{\pi_k}v_{k-1} - v_{\pi_{k,m}} &= \Gamma_{k,m}
  (v_{k-m}-v_{\pi_{k,m}}) + \sum_{i=1}^{m-1} \Gamma_{k,i} \epsilon_{k-i}
\end{align*}
and the result follows by taking the norm and using the fact that for all $i$, $\|\Gamma_{k,i}\|_\infty=\gamma^i$.
\end{proof}

We are now ready to prove the main result of this section.
\begin{proof}[Proof of Theorem~\ref{thm:avi}]
Using the fact that $T$ is a contraction in max-norm, we have:
\begin{align*}
  \norm{v_* -v_k}_\infty &= \norm{v_* - T v_{k-1} + \epsilon_k}_\infty\\
  &\leq \norm{T v_* - T v_{k-1}}_\infty + \epsilon\\
  &\leq \gamma\norm{v_* - v_{k-1}}_\infty+\epsilon.
\end{align*}
Then, by induction on $k$, we have that for all $k \ge 1$,
\begin{align}
  \|v_*-v_k\|_\infty \le \gamma^k \|v_*-v_0\|_\infty +
  \frac{1-\gamma^k}{1-\gamma} \epsilon. \label{eq:vi-norm}
\end{align}
Using Lemma~\ref{lem:vi-norm-rec} and Equation~\eqref{eq:vi-norm} twice, we can conclude by observing that
\begin{align*}
  \|v_* - v_{\pi_{k,m}} \|_\infty &\le \|T v_* - T v_{k-1} \|_\infty + \|T v_{k-1} - v_{\pi_{k,m}} \|_\infty \\
  &  \le  \gamma \|v_*-v_{k-1}\|_\infty  + \gamma^m \|v_{k-m} - v_{\pi_{k,m}}  \|_\infty + \frac{\gamma-\gamma^m}{1-\gamma} \epsilon \\
  & \leq \gamma \left( \gamma^{k-1} \|v_*-v_0\|_\infty + \frac{1-\gamma^{k-1}}{1-\gamma} \epsilon\right)\\
  & ~~~~~~~ + \gamma^m \left( \| v_{k-m}-v_* \|_\infty  + \|v_* - v_{\pi_{k,m}} \|_\infty \right) + \frac{\gamma-\gamma^m}{1-\gamma} \epsilon \\
  & \le \gamma^k \|v_*-v_0\|_\infty + \frac{\gamma-\gamma^k}{1-\gamma} \epsilon \\
  & ~~~~~~~ +  \gamma^m \left( \gamma^{k-m}  \|v_*-v_0\|_\infty + \frac{1-\gamma^{k-m}}{1-\gamma} \epsilon + \|v_* - v_{\pi_{k,m}} \|_\infty  \right) + \frac{\gamma-\gamma^m}{1-\gamma} \epsilon \\
  & = \gamma^m \|v_* - v_{\pi_{k,m}} \|_\infty + 2\gamma^k \|v_*-v_0\|_\infty +
  \frac{2(\gamma-\gamma^k)}{1-\gamma}\epsilon\\
  &\le  \frac 2{1-\gamma^m}\left(\frac{\gamma-\gamma^k}{1-\gamma}\epsilon +\gamma^k \|v_*-v_0\|_\infty \right). \qedhere
\end{align*}
\end{proof}

\section{API algorithms for computing non-stationary policies}
\label{sec:api}

We now present similar results that have a Policy Iteration flavour. Unlike in the previous section where only the output of AVI needed to be changed, improving the bound for an API-like algorithm is slightly more involved. In this section, we describe and analyze two API algorithms that output non-stationary policies with improved performance bounds.

\paragraph{API with a non-stationary policy of growing period}

Following our findings on non-stationary policies AVI, we consider the following variation of API, where at each iteration, instead of computing the value of
the last stationary policy $\pi_k$, we compute that of the periodic non-stationary policy $\pi_{k,k}$
that loops over all the policies $\pi_1,\dots,\pi_k$ generated \emph{from the very start}:
\begin{align*}
v_{k} &\gets v_{\pi_{k,k}} + \epsilon_{k}\\
\pi_{k+1}& \gets \mbox{any element of } \greedy{v_k}
\end{align*}
where the initial (stationary) policy $\pi_{1,1}$ is chosen arbitrarily. Thus, iteration after iteration, the non-stationary policy $\pi_{k,k}$ is made of more and more stationary policies, and this is why we refer to it as having a growing period.
We can prove the following performance bound for this algorithm:
\begin{thm}
\label{thm:api1}
 After $k$ iterations, the loss of running the non-stationary policy $\pi_{k,k}$ instead of the optimal policy $\pi_*$ satisfies:
\[
\|v_*- v_{\pi_{k,k}}\|_\infty \le \frac{2 (\gamma-\gamma^k)}{1-\gamma}\epsilon + \gamma^{k-1} \| v_* - v_{\pi_{1,1}} \|_\infty + 2(k-1) \gamma^{k} V_{\mathrm{max}}.
\]
\end{thm}
When $k$ tends to infinity, this bound tends to $\frac{2\gamma}{1-\gamma}\epsilon$, and is thus again a factor $\frac{1}{1-\gamma}$ better than the original API bound.
\begin{proof}[Proof of Theorem~\ref{thm:api1}]
Using the facts that $T_{k+1,k+1}v_{\pi_{k,k}}=T_{\pi_{k+1}} T_{k,k}v_{\pi_{k,k}}=T_{\pi_{k+1}}v_{\pi_{k,k}}$ and $T_{\pi_{k+1}}v_k \ge T_{\pi_*} v_k$ (since $\pi_{k+1} \in \greedy{v_k}$), we have:
\begin{align*}
& v_* - v_{\pi_{k+1,k+1}}\\
=&~ T_{\pi_*}v_* - T_{k+1,k+1} v_{\pi_{k+1,k+1}} \\
=&~ T_{\pi_*}v_* - T_{\pi_*}v_{\pi_{k,k}} + T_{\pi_*}v_{\pi_{k,k}} - T_{k+1,k+1}v_{\pi_{k,k}} + T_{k+1,k+1}v_{\pi_{k,k}} - T_{k+1,k+1} v_{\pi_{k+1,k+1}}  \\
=&~ \gamma P_{\pi_*}(v_*- v_{\pi_{k,k}}) +  T_{\pi_*}v_{\pi_{k,k}} - T_{\pi_{k+1}} v_{\pi_{k,k}} + \Gamma_{k+1,k+1} (v_{\pi_{k,k}} - v_{\pi_{k+1,k+1}}) \\
=&~ \gamma P_{\pi_*}(v_*- v_{\pi_{k,k}})  + T_{\pi_*}v_k - T_{\pi_{k+1}}v_k + \gamma (P_{\pi_{k+1}}-P_{\pi_*})\epsilon_k + \Gamma_{k+1,k+1} (v_{\pi_{k,k}} - v_{\pi_{k+1,k+1}}) \\
\le&~ \gamma P_{\pi_*}(v_*- v_{\pi_{k,k}})  + \gamma (P_{\pi_{k+1}}-P_{\pi_*})\epsilon_k + \Gamma_{k+1,k+1} (v_{\pi_{k,k}} - v_{\pi_{k+1,k+1}}).
\end{align*}
By taking the norm, and using the facts that $\|v_{\pi_{k,k}}\|_\infty \le V_{\mathrm{max}}$, $\|v_{\pi_{k+1,k+1}}\|_\infty \le V_{\mathrm{max}}$, and $\|\Gamma_{k+1,k+1}\|_\infty=\gamma^{k+1}$, we get:
\[
\| v_* - v_{\pi_{k+1,k+1}} \|_\infty \le \gamma \|v_*- v_{\pi_{k,k}}\|_\infty + 2\gamma \epsilon + 2\gamma^{k+1}V_{\mathrm{max}}.
\]
Finally, by induction on $k$, we obtain:
\[
\|v_*- v_{\pi_{k,k}}\|_\infty \le \frac{2 (\gamma-\gamma^k)}{1-\gamma}\epsilon + \gamma^{k-1} \| v_* - v_{\pi_{1,1}} \|_\infty + 2(k-1) \gamma^{k} V_{\mathrm{max}}. \qedhere
\]
\end{proof}

Though it has an improved asymptotic performance bound, the API algorithm we have
just described has two (related) drawbacks: 1) its finite iteration
bound has a somewhat unsatisfactory term of the form $2(k-1)\gamma^{k}V_{\mathrm{max}}$,
and 2) even when there is no error (when $\epsilon=0$), we cannot
guarantee that, similarly to standard Policy Iteration, it generates a sequence of policies of increasing
values (it is easy to see that in general, we do not have
$v_{\pi_{k+1,k+1}} \ge v_{\pi_{k,k}}$). These two points motivate the
introduction of another API algorithm.

\paragraph{API with a non-stationary policy of fixed period}

We consider now another variation of API parameterized by $m \ge 1$, that iterates as follows for $k\geq m$:
\begin{align*}
v_{k} &\gets  v_{\pi_{k,m}} + \epsilon_{k}\\
\pi_{k+1} &\gets \mbox{any element of } \greedy{v_k}
\end{align*}
where the initial non-stationary policy $\pi_{m,m}$ is built from a
sequence of $m$ arbitrary stationary policies $\pi_1,\pi_2,\cdots,\pi_m$. Unlike the previous API
algorithm, the non-stationary policy $\pi_{k,m}$ here only involves the last $m$ greedy stationary policies
instead of all of them, and is thus of fixed period. 
This is a strict generalization of the standard API algorithm, with which it coincides when $m=1$.
For this algorithm, we can prove the following performance bound:
\begin{thm}
\label{thm:api2}
For all $m$, for all $k \ge m$, the loss of running the non-stationary policy $\pi_{k,m}$ instead of the optimal policy $\pi_*$ satisfies:
\begin{align*}
  \|v_*-v_{\pi_{k,m}}\|_\infty \le \gamma^{k-m} \| v_* - v_{\pi_{m,m}} \|_\infty +
  \frac{2(\gamma-\gamma^{k+1-m})}{(1-\gamma)(1-\gamma^m)}\epsilon.
\end{align*}
\end{thm}
When $m=1$ and $k$ tends to infinity, we recover exactly the bound of Theorem~\ref{thm:classic-bound}. When $m > 1$ and $k$ tends to infinity, this bound coincides with that of Theorem~\ref{thm:avi} for our non-stationary version of AVI: it is a factor $\frac{1-\gamma^m}{1-\gamma}$ better than the standard bound of Theorem~\ref{thm:classic-bound}.

The rest of this section develops the proof of this performance bound.
A central argument of our proof is the following lemma, which shows that similarly to the standard API, our new algorithm has an (approximate) policy improvement property.
\begin{lemma}
\label{lem:monot}
At each iteration of the algorithm, the value $v_{\pi_{k+1,m}}$ of the non-stationary policy
$$
\pi_{k+1,m} ~=~ \pi_{k+1}\  \pi_{k}\  \dots\  \pi_{k+2-m}\  \pi_{k+1}\  \pi_{k}\  \dots\  \pi_{k-m+2} \dots
$$ 
cannot be much worse than the value $v_{\pi'_{k,m}}$ of the non-stationary policy
$$
\pi'_{k,m} ~=~ \pi_{k-m+1}\  \pi_{k}\  \dots\  \pi_{k+2-m}\  \pi_{k-m+1}\  \pi_{k}\  \dots\  \pi_{k-m+2} \dots
$$
in the precise following sense:
$$
v_{\pi_{k+1,m}} \ge v_{\pi'_{k,m}} - \frac{2\gamma}{1-\gamma^m} \epsilon.
$$
\end{lemma}
The policy $\pi'_{k,m}$ differs from $\pi_{k+1,m}$ in that every $m$ steps, it chooses the oldest policy $\pi_{k-m+1}$ instead of the newest one $\pi_{k+1}$. Also $\pi'_{k,m}$ is related to $\pi_{k,m}$ as follows: $\pi'_{k,m}$ takes the first action according to $\pi_{k-m+1}$ and then runs $\pi_{k,m}$; equivalently, since $\pi_{k,m}$ loops over $\pi_k \pi_{k-1} \dots \pi_{k-m+1}$, $\pi'_{k,m}=\pi_{k-m+1}\pi_{k,m}$ can be seen as a 1-step right rotation of $\pi_{k,m}$. When there is no error (when $\epsilon=0$), this shows that the new policy $\pi_{k+1,m}$ is better than a ``rotation'' of $\pi_{k,m}$. When $m=1$, $\pi_{k+1,m}=\pi_{k+1}$ and $\pi'_{k,m}=\pi_{k}$ and we thus recover the well-known (approximate) policy improvement theorem for standard API~(see for instance~\cite[Lemma 6.1]{ndp}).

\begin{proof}[Proof of Lemma~\ref{lem:monot}]
Since $\pi'_{k,m}$ takes the first action with respect to $\pi_{k-m+1}$ and then runs $\pi_{k,m}$, we have $v_{\pi'_{k,m}}=T_{\pi_{k-m+1}} v_{\pi_{k,m}}$.
Now, since $\pi_{k+1} \in \greedy{v_k}$, we have $T_{\pi_{k+1}}v_k \ge T_{\pi_{k-m+1}} v_k$ and
\begin{align*}
v_{\pi'_{k,m}}-v_{\pi_{k+1,m}} &= T_{\pi_{k-m+1}} v_{\pi_{k,m}} - v_{\pi_{k+1,m}} \\
& = T_{\pi_{k-m+1}} v_k - \gamma P_{\pi_{k-m+1}}\epsilon_k - v_{\pi_{k+1,m}} \\
& \le T_{\pi_{k+1}} v_k - \gamma P_{\pi_{k-m+1}}\epsilon_k - v_{\pi_{k+1,m}} \\
& = T_{\pi_{k+1}} v_{\pi_{k,m}} + \gamma(P_{\pi_{k+1}}-P_{\pi_{k-m+1}})\epsilon_k - v_{\pi_{k+1,m}} \\
& = T_{\pi_{k+1}} T_{k,m} v_{\pi_{k,m}} - T_{k+1,m} v_{\pi_{k+1,m}} + \gamma(P_{\pi_{k+1}}-P_{\pi_{k-m+1}})\epsilon_k \\
& = T_{k+1,m} T_{\pi_{k-m+1}} v_{\pi_{k,m}} - T_{k+1,m} v_{\pi_{k+1,m}} + \gamma(P_{\pi_{k+1}}-P_{\pi_{k-m+1}})\epsilon_k \\
& = \Gamma_{k+1,m}(T_{\pi_{k-m+1}} v_{\pi_{k,m}} - v_{\pi_{k+1,m}}) + \gamma(P_{\pi_{k+1}}-P_{\pi_{k-m+1}})\epsilon_k \\
&= \Gamma_{k+1,m}(v_{\pi'_{k,m}} - v_{\pi_{k+1,m}}) + \gamma(P_{\pi_{k+1}}-P_{\pi_{k-m+1}})\epsilon_k.
\end{align*}
from which we deduce that:
\begin{align*}
v_{\pi'_{k,m}} - v_{\pi_{k+1,m}} \le (I-\Gamma_{k+1,m})^{-1}\gamma(P_{\pi_{k+1}}-P_{\pi_{k-m+1}})\epsilon_k
\end{align*}
and the result follows by using the facts that $\| \epsilon_k \|_\infty \le \epsilon$ and $\|(I-\Gamma_{k+1,m})^{-1}\|_\infty=\frac{1}{1-\gamma^m}$.
\end{proof}

We are now ready to prove the main result of this section.
\begin{proof}[Proof of Theorem~\ref{thm:api2}]
Using the facts that 1) $T_{k+1,m+1}v_{\pi_{k,m}}=T_{\pi_{k+1}} T_{k,m}v_{\pi_{k,m}}=T_{\pi_{k+1}}v_{\pi_{k,m}}$ and 2) $T_{\pi_{k+1}}v_k \ge T_{\pi_*} v_k$ (since $\pi_{k+1} \in \greedy{v_k}$),  we have for $k \ge m$,
\begin{align}
&v_* - v_{\pi_{k+1,m}} \nonumber\\
 =&~ T_{\pi_*}v_* - T_{k+1,m} v_{\pi_{k+1,m}} \notag\\
=&~ T_{\pi_*}v_* - T_{\pi_*}v_{\pi_{k,m}} + T_{\pi_*}v_{\pi_{k,m}} - T_{k+1,m+1}v_{\pi_{k,m}} + T_{k+1,m+1}v_{\pi_{k,m}} - T_{k+1,m} v_{\pi_{k+1,m}}   \notag\\
=&~ \gamma P_{\pi_*} (v_* -  v_{\pi_{k,m}}) + T_{\pi_*}v_{\pi_{k,m}} - T_{\pi_{k+1}}v_{\pi_{k,m}} + \Gamma_{k+1,m}(T_{\pi_{k-m+1}} v_{\pi_{k,m}} - v_{\pi_{k+1,m}}) \notag\\
\le&~  \gamma P_{\pi_*} (v_* -  v_{\pi_{k,m}}) + T_{\pi_*}v_k - T_{\pi_{k+1}}v_k + \gamma (P_{\pi_{k+1}}-P_{\pi_*})\epsilon_k + \Gamma_{k+1,m}(T_{\pi_{k-m+1}} v_{\pi_{k,m}} - v_{\pi_{k+1,m}}) \notag\\
 \le&~ \gamma P_{\pi_*} (v_* -  v_{\pi_{k,m}}) + \gamma (P_{\pi_{k+1}}-P_{\pi_*})\epsilon_k +  \Gamma_{k+1,m}(T_{\pi_{k-m+1}} v_{\pi_{k,m}} - v_{\pi_{k+1,m}}). \label{eq1}
\end{align}
Consider the policy $\pi'_{k,m}$ defined in Lemma~\ref{lem:monot}. Observing as in the beginning of the proof of Lemma~\ref{lem:monot} that $T_{\pi_{k-m+1}} v_{\pi_{k,m}}=v_{\pi'_{k,m}}$, Equation~\eqref{eq1} can be rewritten as follows:
\begin{align*}
v_* - v_{\pi_{k+1,m}} & \le \gamma P_{\pi_*} (v_* -  v_{\pi_{k,m}}) + \gamma (P_{\pi_{k+1}}-P_{\pi_*})\epsilon_k + \Gamma_{k+1,m} (v_{\pi'_{k,m}}-v_{\pi_{k+1,m}}).
\end{align*}
By using the facts that $v_* \ge v_{\pi_{k,m}}$, $v_* \ge v_{\pi_{k+1,m}}$ and Lemma~\ref{lem:monot}, we get
\begin{align*}
\|v_*-v_{\pi_{k+1,m}}\|_\infty  & \le \gamma \| v_* -  v_{\pi_{k,m}} \|_\infty + 2 \gamma \epsilon + \frac{\gamma^m(2\gamma \epsilon)}{1-\gamma^m} \\
& = \gamma \| v_* -  v_{\pi_{k,m}} \|_\infty + \frac{2\gamma}{1-\gamma^m}  \epsilon.
\end{align*}
Finally, we obtain by induction that for all $k \ge m$, 
\begin{align*}
\|v_*-v_{\pi_{k,m}}\|_\infty  \le \gamma^{k-m} \| v_* -  v_{\pi_{m,m}} \|_\infty + \frac{2(\gamma-\gamma^{k+1-m})}{(1-\gamma)(1-\gamma^m)}\epsilon. ~~~~~~\qedhere
\end{align*}
\end{proof}

\section{Discussion, conclusion and future work}
\label{sec:discussion}

We recalled in Theorem~\ref{thm:classic-bound} the standard
performance bound when computing an approximately optimal stationary
policy with the standard AVI and API algorithms. After arguing that this bound
is tight -- in particular by providing an original argument for AVI --
we proposed three new dynamic programming algorithms (one based on AVI
and two on API) that output non-stationary policies for which the
performance bound can be significantly reduced (by a factor
$\frac{1}{1-\gamma}$).

From a bibliographical point of view, it is the work of \cite{kakade}
that made us think that non-stationary policies may lead to better
performance bounds. In that work, the author considers problems with a
finite-horizon $T$ for which one computes \emph{non-stationary}
policies with performance bounds in $O(T \epsilon)$, and
infinite-horizon problems for which one computes \emph{stationary}
policies with performance bounds in
$O(\frac{\epsilon}{(1-\gamma)^2})$.  Using the informal equivalence of
the horizons $T \simeq \frac{1}{1-\gamma}$ one sees that
non-stationary policies look better than
stationary policies.  In \cite{kakade}, non-stationary policies
are only computed in the context of finite-horizon (and thus
non-stationary) problems;   the fact that non-stationary policies can
also be useful in an infinite-horizon stationary context is to our
knowledge completely new.

The best performance improvements are obtained when our algorithms
consider periodic non-stationary policies of which the period grows to
infinity, and thus require an infinite memory, which may look like a practical limitation. However, in two of the proposed algorithm, a parameter $m$
allows to make a trade-off between the quality of approximation
$\frac{2\gamma}{(1-\gamma^m)(1-\gamma)}\epsilon$ and the amount of
memory $O(m)$ required. In practice, it is easy to see that by
choosing $m=\left\lceil \frac{1}{1-\gamma} \right\rceil$, that is a
memory that scales linearly with the horizon (and thus the difficulty)
of the problem, one can get a performance bound of\footnote{With this
  choice of $m$, we have $m \ge \frac{1}{\log{1/\gamma}}$ and thus
  $\frac{2}{1-\gamma^m} \le \frac{2}{1-e^{-1}}\le 3.164$.}
$\frac{2\gamma}{(1-e^{-1})(1-\gamma)}\epsilon \le \frac{3.164
  \gamma}{1-\gamma}\epsilon $.




We conjecture that our asymptotic bound of
$\frac{2\gamma}{1-\gamma}\epsilon$, and the non-asymptotic bounds of
Theorems~\ref{thm:avi} and~\ref{thm:api2} are tight. The actual proof of this conjecture is left for
future work. Important recent works of the literature involve studying
performance bounds when the errors are controlled in $L_p$ norms
instead of
max-norm~\cite{munos2003,Munos_SIAM07,Munos_JMLR08,antos2008learning,farahmand2009regularized,acml2010,Lazaric_JMLR2011_a}
which is natural when supervised learning algorithms are used to
approximate the evaluation steps of AVI and API. Since our proof are
based on componentwise bounds like those of the pioneer works in this
topic~\cite{munos2003,Munos_SIAM07}, we believe that the extension
of our analysis to $L_p$ norm analysis is straightforward.  Last but not
least, an important research direction that we plan to follow consists
in revisiting the many implementations of AVI and API for building
stationary policies (see the list in the introduction), turn them into
algorithms that look for non-stationary policies and study them
precisely analytically as well as empirically.



\bibliographystyle{plain}
\bibliography{biblio}

\begin{thebibliography}{10}

\bibitem{antos2008learning}
A.~Antos, {Cs}. Szepesv{\'a}ri, and R.~Munos.
\newblock {Learning near-optimal policies with Bellman-residual minimization
  based fitted policy iteration and a single sample path}.
\newblock {\em Machine Learning}, 71(1):89--129, 2008.

\bibitem{gheshlaghidpp}
M.~Gheshlaghi Azar, V.~Gómez, and H.J. Kappen.
\newblock {Dynamic Policy Programming with Function Approximation}.
\newblock In {\em 14th International Conference on Artificial Intelligence and
  Statistics (AISTATS)}, volume~15, Fort Lauderdale, FL, USA, 2011.

\bibitem{bertsekas2011}
D.P. Bertsekas.
\newblock Approximate policy iteration: a survey and some new methods.
\newblock {\em Journal of Control Theory and Applications}, 9:310--335, 2011.

\bibitem{ndp}
D.P. Bertsekas and J.N. Tsitsiklis.
\newblock {\em Neuro-Dynamic Programming}.
\newblock Athena Scientific, 1996.

\bibitem{busoniu2011least-squares}
L.~Busoniu, A.~Lazaric, M.~Ghavamzadeh, R.~Munos, R.~Babuska, and B.~De
  Schutter.
\newblock {Least-squares methods for Policy Iteration}.
\newblock In M.~Wiering and M.~van Otterlo, editors, {\em Reinforcement
  Learning: State of the Art}. Springer, 2011.

\bibitem{ernst2005tree}
D.~Ernst, P.~Geurts, and L.~Wehenkel.
\newblock {Tree-based batch mode reinforcement learning}.
\newblock {\em Journal of Machine Learning Research (JMLR)}, 6, 2005.

\bibitem{Even-dar05y.:planning}
E.~Even-dar.
\newblock {Planning in pomdps using multiplicity automata}.
\newblock In {\em Uncertainty in Artificial Intelligence (UAI}, pages 185--192,
  2005.

\bibitem{farahmand2009regularized}
A.M. Farahmand, M.~Ghavamzadeh, {Cs}. Szepesv{\'a}ri, and S.~Mannor.
\newblock Regularized policy iteration.
\newblock {\em Advances in Neural Information Processing Systems}, 21:441--448,
  2009.

\bibitem{FaMuSz10}
A.M. Farahmand, R.~Munos, and {Cs}. Szepesv{\'a}ri.
\newblock Error propagation for approximate policy and value iteration
  (extended version).
\newblock In {\em NIPS}, December 2010.

\bibitem{gabillon:hal-00644935}
V.~Gabillon, A.~Lazaric, M.~Ghavamzadeh, and B.~Scherrer.
\newblock {Classification-based Policy Iteration with a Critic}.
\newblock In {\em {International Conference on Machine Learning (ICML)}}, pages
  1049--1056, Seattle, {\'E}tats-Unis, June 2011.

\bibitem{Gordon95}
G.J. Gordon.
\newblock {Stable Function Approximation in Dynamic Programming}.
\newblock In {\em ICML}, pages 261--268, 1995.

\bibitem{guestrin2001}
C.~Guestrin, D.~Koller, and R.~Parr.
\newblock {Max-norm projections for factored MDPs}.
\newblock In {\em International Joint Conference on Artificial Intelligence},
  volume 17-1, pages 673--682, 2001.

\bibitem{Guestrinjair2003}
C.~Guestrin, D.~Koller, R.~Parr, and S.~Venkataraman.
\newblock {Efficient Solution Algorithms for Factored MDPs}.
\newblock {\em Journal of Artificial Intelligence Research (JAIR)},
  19:399--468, 2003.

\bibitem{kakade}
S.M. Kakade.
\newblock {\em On the Sample Complexity of Reinforcement Learning}.
\newblock PhD thesis, University College London, 2003.

\bibitem{KakadeL02}
S.M. Kakade and J.~Langford.
\newblock {Approximately Optimal Approximate Reinforcement Learning}.
\newblock In {\em International Conference on Machine Learning (ICML)}, pages
  267--274, 2002.

\bibitem{lagoudakis2003least}
M.G. Lagoudakis and R.~Parr.
\newblock Least-squares policy iteration.
\newblock {\em Journal of Machine Learning Research (JMLR)}, 4:1107--1149,
  2003.

\bibitem{Lazaric_JMLR2011_a}
A.~Lazaric, M.~Ghavamzadeh, and R.~Munos.
\newblock {Finite-Sample Analysis of Least-Squares Policy Iteration}.
\newblock {\em To appear in Journal of Machine learning Research (JMLR)}, 2011.

\bibitem{acml2010}
O.A. Maillard, R.~Munos, A.~Lazaric, and M.~Ghavamzadeh.
\newblock {Finite Sample Analysis of {B}ellman Residual Minimization}.
\newblock In Masashi Sugiyama and Qiang Yang, editors, {\em Asian Conference on
  Machine Learpning. JMLR: Workshop and Conference Proceedings}, volume~13,
  pages 309--324, 2010.

\bibitem{munos2003}
R.~Munos.
\newblock {Error Bounds for Approximate Policy Iteration}.
\newblock In {\em International Conference on Machine Learning (ICML)}, pages
  560--567, 2003.

\bibitem{Munos_SIAM07}
R.~Munos.
\newblock {Performance Bounds in {Lp} norm for Approximate Value Iteration}.
\newblock {\em SIAM J. Control and Optimization}, 2007.

\bibitem{Munos_JMLR08}
R.~Munos and Cs. Szepesv\'ari.
\newblock Finite time bounds for sampling based fitted value iteration.
\newblock {\em Journal of Machine Learning Research (JMLR)}, 9:815--857, 2008.

\bibitem{petrik:2008}
M.~Petrik and B.~Scherrer.
\newblock {Biasing Approximate Dynamic Programming with a Lower Discount
  Factor}.
\newblock In {\em {Twenty-Second Annual Conference on Neural Information
  Processing Systems -NIPS 2008}}, Vancouver, Canada, 2008.

\bibitem{pineau2003point}
J.~Pineau, G.J. Gordon, and S.~Thrun.
\newblock {Point-based value iteration: An anytime algorithm for POMDPs}.
\newblock In {\em International Joint Conference on Artificial Intelligence},
  volume~18, pages 1025--1032, 2003.

\bibitem{puterman}
M.~Puterman.
\newblock {\em Markov {D}ecision {P}rocesses}.
\newblock Wiley, New York, 1994.

\bibitem{singh94}
S.~Singh and R.~Yee.
\newblock {An Upper Bound on the Loss from Approximate Optimal-Value
  Functions}.
\newblock {\em Machine Learning}, 16-3:227--233, 1994.

\bibitem{thierylpi}
C.~Thiery and B.~Scherrer.
\newblock {Least-Squares $\lambda$ Policy Iteration: Bias-Variance Trade-off in
  Control Problems}.
\newblock In {\em {International Conference on Machine Learning}}, Haifa,
  Israel, 2010.

\bibitem{TsitsiklisR96}
J.N. Tsitsiklis and B.~Van Roy.
\newblock {Feature-Based Methods for Large Scale Dynamic Programming}.
\newblock {\em Machine Learning}, 22(1-3):59--94, 1996.

\end{thebibliography}

\end{document}